\documentclass{article}
\usepackage{arxiv}
\pdfoutput=1
\usepackage[utf8]{inputenc} % allow utf-8 input
\usepackage{hyperref}       % hyperlinks
\usepackage{nicefrac}       % compact symbols for 1/2, etc.
\usepackage{microtype}      % microtypography

\usepackage{mathtools}
\usepackage{amsmath,amssymb}
\usepackage{amsthm}
\usepackage{amsfonts}
\usepackage{array}
\usepackage{balance}
\usepackage{bm}
\usepackage{booktabs}
\usepackage{color}
\usepackage{dsfont}
\usepackage{dirtytalk}
\usepackage{epsfig, epstopdf}
\usepackage{graphics}
\usepackage{graphicx}
\usepackage{multirow, rotating}
\usepackage{pstricks}
\usepackage{pifont}
\usepackage{subcaption}
\usepackage[flushleft]{threeparttable}
\usepackage{xcolor}
\usepackage{xfrac}
\usepackage{verbatim}
\usepackage[noend]{algorithmic,algorithm2e}

\newcolumntype{M}[1]{>{\centering\arraybackslash}m{#1}}

\newcommand{\xmark}{\ding{55}}
\newcommand{\cmark}{\ding{51}}%

\newcommand{\WL}{\textbf{WL}}

\makeatletter
\newcommand\footnoteref[1]{\protected@xdef\@thefnmark{\ref{#1}}\@footnotemark}
\makeatother

\newtheorem{Lemma}{Lemma}
\newtheorem{Definition}{Definition}

\usepackage{hyperref}
\hypersetup{
    colorlinks=true,
    linkcolor=blue,
    filecolor=magenta,      
    urlcolor=cyan,
}

\title{Hierarchical and Unsupervised Graph Representation Learning with Loukas's Coarsening}

\newcommand*\samethanks[1][\value{footnote}]{\footnotemark[#1]}
\author{
  Louis Béthune\thanks{Corresponding author.}~\thanks{Equal contribution}\\
  Laboratoire de Physique, UMR 5672\\
  CNRS, ENS de Lyon, UCB Lyon 1\\
  Lyon, France\\
  \texttt{louis.bethune@ens-lyon.fr} \\
   \And
  Yacouba Kaloga\samethanks\\
  Laboratoire de Physique, UMR 5672\\
  CNRS, ENS de Lyon, UCB Lyon 1\\
  Lyon, France\\
  \texttt{yacouba.kaloga@ens-lyon.fr} \\
  \And
  Pierre Borgnat\\
  Laboratoire de Physique, UMR 5672\\
  CNRS, ENS de Lyon, UCB Lyon 1\\
  Lyon, France\\
  \texttt{pierre.borgnat@ens-lyon.fr} \\
  \And
  Aurélien Garivier\\
UMPA UMR 5669 and LIP UMR 5668\\
  CNRS, ENS de Lyon, UCB Lyon 1\\
  Lyon, France\\
  \texttt{aurelien.garivier@ens-lyon.fr} \\
  \And
  Amaury Habrard\\
  Laboratoire Hubert Curien, UMR 5516\\
  CNRS, University of Lyon, UJM-Saint-Etienne\\
  Saint-Etienne, France\\
  \texttt{amaury.habrard@univ-st-etienne.fr} 
}

\begin{document}
\maketitle

\begin{abstract}
We propose a novel algorithm for unsupervised graph representation learning with attributed graphs. It combines three advantages addressing some current limitations of the literature:
i) The model is inductive: it can embed new graphs without re-training in the presence of new data;
ii) The method takes into account both micro-structures and macro-structures by looking at the attributed graphs at different scales;
iii) The model is end-to-end differentiable: it is a building block that can be plugged into deep learning pipelines and allows for back-propagation.
We show that combining a coarsening method having strong theoretical guarantees with mutual information maximization suffices to produce high quality embeddings. We evaluate them on classification tasks with common benchmarks of the literature. We show that our algorithm is competitive with state of the art among unsupervised graph representation learning methods.
\end{abstract}

% Keywords

\section{Introduction}

Graphs are a canonical way of representing objects and relationships among them. They have proved remarkably well suited in many fields such as chemistry, biology, social sciences or computer science in general. The connectivity information (edges) is often completed by discrete labels or continuous attributes on nodes, resulting in so-called \emph{attributed graphs}. Many real-life problems involving high dimensional objects and their links can be modeled using attributed graphs.

Machine learning offers several ways to solve problems such as classification, clustering or inference, provided that a sufficient amount of training examples is available. Yet, the most classical frameworks are devoted to data living in regular spaces (e.g. vector spaces), and they are not suitable to deal with attributed graphs. One way to overcome this issue is to represent or encode the attributed graphs in such a way that usual machine learning approaches are efficient. A recent take on that is known as \emph{graph representation learning} \cite{hamilton2017representation}: the graphs are embedded in a fixed dimensional latent space such that similar graphs share similar embeddings.

Three properties are desirable in order for a method of attributed graph representation learning to be widely applicable and expressive enough.
We expect a method to be: \textbf{I. Unsupervised} because labels are expensive, and not always available; \textbf{II. Inductive} so that computing the embedding of an unseen graph (not belonging to the training set) can be done on the fly (in contrast to transductive methods);
\textbf{III. Hierarchical} so as to take into account properties on both local and global scales; indeed, structured  information in graphs can reside at various scales, from small neighborhoods to the entire graph.

In order to obtain these three desirable properties for attributed graphs representation learning, the present work introduces a new \textbf{Hierarchical Graph2Vec (HG2V)} model. Like Graph2Vec \cite{narayanan2017graph2vec} with which it shares some similarities, it is based on the maximization of some mutual information. Thanks to a proper use of coarsening, as proposed by  Loukas \cite{loukas2018graph}, it is hierarchical and incorporates information at all scales, from micro-structures like node neighborhoods up to macro-structures (communities, bridges), by considering a pyramid of attributed graphs of decreasing size.

The article is organised as follows. Section~\ref{s:related-work} presents some related work in the literature. In Section~\ref{s:definition-background}, we introduce the notation and the required background. Section~\ref{s:contribution} is dedicated to the detailed presentation of our main contribution: the Hierarchical Graph2Vec method. In Section~\ref{s:evalution}, an experimental study is reported that demonstrates the effectiveness of the framework for various tasks.

\section{Related Work}
\label{s:related-work}

Graph Representation learning is related to a large spectrum of works, from kernel algorithms to graph neural networks.

\noindent \textbf{Kernels methods.} Graph kernels have become a well established and a widely used technique for learning graph representations~\cite{togninalli2019wasserstein,vishwanathan2010graph}. They use handcrafted similarity measures between every pair of graphs. Some are restricted to discrete labels \cite{10.5555/2984093.2984279,shervashidze2011weisfeiler}, while others can handle continuous attributes \cite{NIPS2013_5155,ICML2012Kriege_542,Morris_2016}. The main drawback of kernel methods is the computational burden of building and storing the kernel matrix, which has quadratic complexity unless using approximation techniques.

\noindent \textbf{Infomax Principle} hypothesizes that good representations maximize mutual information between the input data and its embedding. Deep Graph Infomax \cite{velickovic2018deep} and GraphSAGE \cite{hamilton2017inductive} rely on negative sampling to build an estimator of Mutual Information (MI). It is used to produce node embeddings for solving a classification task. InfoGraph \cite{Sun2020InfoGraph} uses the same estimator to produce embeddings for entire attributed graphs. % L54: reviewer 1
  
\noindent \textbf{Graph2Vec} \cite{narayanan2017graph2vec} uses the same MI estimator in combination with Wesfeiler-Lehman procedure (see Section~\ref{s:definition-background}) to produce graph embeddings. It was originally inspired by languages models (especially Word2Vec) and considers node embedding as the vocabulary used to ``write'' a graph. %The method is transductive, it uses discrete labels and is inefficient for continuous attributes due to the use of a discrete hash function.

\noindent \textbf{Graph coarsening.} The aim of graph coarsening is to produce a sequence of graphs of decreasing sizes; it can be done by node pooling, as with Loukas' algorithm \cite{loukas2018graph}, or by node decimation, like for example Kron reduction \cite{Dorfler_2013,bianchi2019hierarchical}. Coarsening can be combined with edge sparsification \cite{bravohermsdorff2019unifying,Dorfler_2013}, so as to reduce the density of coarsened graphs. In another context, DiffPool \cite{NIPS2018_7729} performs graph coarsening using clustering but it learns the pooling function specific to each task in a supervised manner. MinCutPool \cite{bianchi2020spectral} also relies on clustering by optimizing a relaxed version of minCut objective, and provides high quality partitioning of the graph.   
  
\noindent \textbf{Graph Neural Networks.} Developed after the renewed interest in Neural Networks, they are known to offer interesting graph embedding methods \cite{hamilton2017representation}, in addition to solving several graph-related task, see \cite{wu2019comprehensive} (and references therein) for a survey, and specifically the popular Chebyshev GNN~\cite{defferrard2016convolutional}, GCN \cite{kipf2017semi} and GAT  \cite{velickovic2018graph}. Stacking them to increase the receptive field may raise scalability issues~\cite{luan2019break}.  

Still, it has been shown that some easy problems on graphs (diameter approximation, shortest paths, cycle detection, s-t cut) cannot be solved by a GNN of insufficient depth \cite{Loukas2020What}. By combining GNNs with coarsening, e.g. \cite{loukas2018graph}, and node pooling, e.g. \cite{bianchi2019hierarchical,Gama_2019,NIPS2018_7729}, those impossibility results no longer apply. This combination thus helps in designing a method encompassing all structural properties of graphs.  

\begin{table}[t]
%\raggedright
\hspace*{-0.0cm}
\centering
\begin{threeparttable}
\begin{tabular}{|M{3.6cm}|M{1.8cm}M{1.8cm}M{1.8cm}M{2.3cm}M{2.3cm}|}
\hline
\textbf{Method} & \textbf{Continuous attributes} & \textbf{Complexity}  (training) & \textbf{Complexity} (inference) & \textbf{End-to-end differentiable} & \textbf{Supervised}\\
\hline
\hline
Kernel methods, e.g. \WL-OA \cite{kriege2016valid}, \textbf{WWL} \cite{togninalli2019wasserstein} & \cmark & $\mathcal{O}(N^2)$\tnote{*} & $\mathcal{O}(N)$\tnote{*} & \xmark & \xmark\\
\hline
Graph2Vec \cite{narayanan2017graph2vec} & \xmark & $\mathcal{O}(N)$ & \xmark & \xmark & \xmark\\
\hline
GIN \cite{xu2018powerful}, DiffPool \cite{NIPS2018_7729}, MinCutPool \cite{bianchi2020spectral} & \cmark & $\mathcal{O}(N)$ & $\mathcal{O}(1)$ & \cmark & \cmark\\
%\hline
%Continuous Graph2Vec (Sec.~\ref{s:evalution}) & \cmark & $\mathcal{O}(N)$ & \xmark & \cmark & \cmark\\
\hline
\textbf{HG2V} (Sec.~\ref{s:contribution}), Infograph \cite{Sun2020InfoGraph} &\cmark & $\mathcal{O}(N)$ & $\mathcal{O}(1)$ & \cmark & \xmark\\
\hline
\end{tabular}
\begin{tablenotes}
\item[*] can be improved with Nystrom approximation or Random Fourier Features.
\end{tablenotes}
\end{threeparttable}
\smallskip
\caption{Key properties of  methods (related or proposed) for graph embedding. $N$ is the number of graphs. Symbol \xmark\text{  }for \textbf{Complexity} (inference) means the method is transductive (and not inductive) and one needs to use the same time as for training. Symbol \cmark\text{  } for \textbf{Supervised} means labels are required to learn a representation (by back-propagating classification loss). %\vspace*{-8mm}
}
\label{tab:advantages}
\end{table}

\section{Definitions and background}
\label{s:definition-background}

The proposed method is inspired by Graph2Vec \cite{narayanan2017graph2vec}, which was itself built upon Negative Sampling and the Weisfeiler-Lehman (\textbf{WL}) algorithm \cite{Weisfeiler1968ReductionOA}. The present section recalls the fundamental ideas of those algorithms. The Weisfeiler-Lehman method produces efficient descriptors of the topology of the graph, which allows to create embeddings of good quality through the maximization of Mutual Information.

\begin{Definition}
An attributed graph is a tuple $(V,A,Z)$, where $V$ is the set of nodes, $A\in\mathbb{R}^{|V|\times |V|}$ is a weighted adjacency matrix, and $Z:V\rightarrow\mathbb{R}^n$ is the function that maps each node $u\in V$ to its attribute vector $Z(u)$. Let $\mathbb{G}$ be the space of attributed graphs.
\end{Definition}

\begin{Definition}
A graph embedding is a function $\mathcal{E}:\mathbb{G}\rightarrow\mathbb{R}^d$ that maps each attributed graph to a vector in the latent space $\mathbb{R}^d$ for some non-negative integer $d$.
\end{Definition}

\subsection{Weisfeiler-Lehman procedure (WL)} The seminal paper \cite{Weisfeiler1968ReductionOA} proposes an algorithm initially created in an attempt to solve graph isomorphism problem (whether or not Graph Isomorphism problem belongs to P is still an open problem). It maps the original graph, with discrete node labels, onto a sequence of labelled graphs, by repeatedly applying the same deterministic operator, as sketched in Fig.~\ref{fig:weisfeilerlehman}. The sequence of node representations generated at each iteration can be used to distinguish between two graphs\footnote{Although it exists different graphs producing the same sequence.}. The  method can be used to build efficient kernels for graph learning \cite{shervashidze2011weisfeiler}. In the following, we will use the \WL-Optimal Assignment kernel \cite{kriege2016valid} as state-of-the-art. The procedure to generate the labels is the following:  
\begin{align} 
x^0(u)&=Z(u)\\
x^{l+1}(u)&=\text{hash}\Big(\big\{x^l(v) | v\in\mathcal{N}(u)\cup\{u\}\big\}\Big)
\label{eq:WLhashing}
\end{align}
where $\mathcal{N}(u)$ is the set of neighbours of $u$.

The hashing function has to be injective in order to distinguish between different rooted subtrees. The notation $\{\}$ emphasises the fact that the output only depends on the unordered set of labels (``bag of labels''). The procedure is iterated on $l$ up to a (user defined) depth $L$: the deeper, the wider the neighborhood used to distinguish between graphs. 

By definition, the label $x^l(u)$ of a node $u$ at the $l$-th level depends only on the labels of the nodes at distance at most $l$ from $u$. The output is invariant under node permutation, and hence is the same for isomorphic graphs. If graph $g_i$ contains $N_i$ nodes, then it produces at most $N_i$ new labels per level, for a maximal number of $N_iL$ new labels at the end.

\begin{figure}[b]
    \centering
    \includegraphics[width=13cm,height=29cm,keepaspectratio]{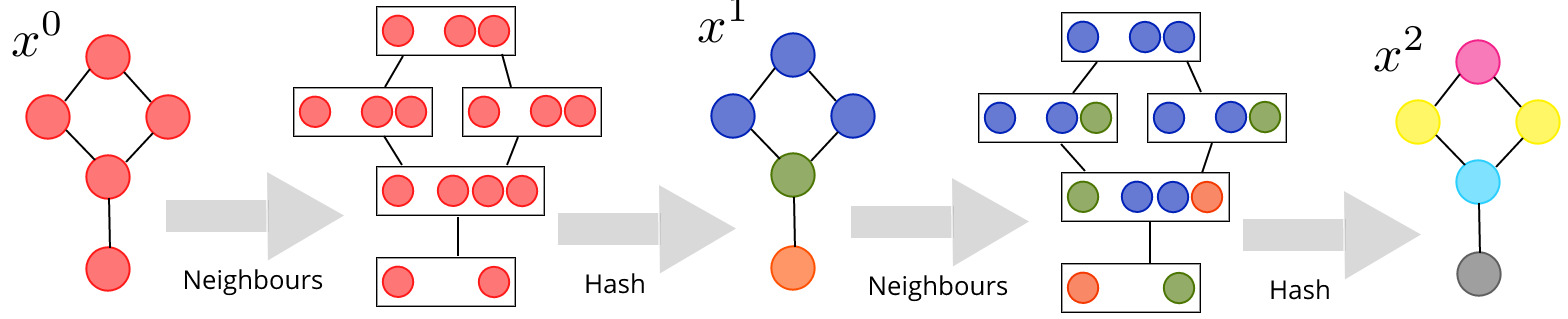}
    \caption{Two iterations of the \WL\text{ }algorithm, starting with uniform labels.}
    \label{fig:weisfeilerlehman}
\end{figure}

\subsection{Negative Sampling and Mutual Information } Many problems in machine learning amount to learn a distribution $P_{Y|X}$ given training examples $(x,y)\sim P_{X,Y}$. It usually requires the usage of a \textit{Softmax} function, implying the expensive computation of a normalization constant: this is unfeasible when the domain of $Y$ is too big (for example in the context of natural language processing, given the size of vocabulary). Hence, Word2Vec \cite{word2vec} replaces it by a sequence of binary classification problems, where a discriminator $T$ learns to distinguish between real and fake pairs. Instead of modeling the conditional distribution $P_{Y|X}$, the binary classifier learns to distinguish between examples coming from the joint distribution $P_{X,Y}$ and coming from the product distribution $P_X\otimes P_Y$.
Negative Sampling can be used to build the Jensen-Shannon estimator of \textbf{Mutual Information (MI)} between $X$ and $Y$ (see \cite{hjelm2018learning} for the first usage of this estimator, see \cite{NIPS2016_6066} for its construction and  \cite{melamud-goldberger-2017-information} for other insights):
\begin{equation}
\label{eq:cross_entropy}
    \widehat{\mathcal{I}}^{\text{(JSD)}}(X,Y)=-\mathbb{E}_{(x,y)\sim P_{XY}}\log{\sigma{(T_{\theta}(x,y)})}-\mathbb{E}_{(x,y)\sim P_X\otimes P_Y}\log{\sigma{(-T_{\theta}(x,y))}}
\end{equation}
where $T_{\theta}:\mathcal{X}\times\mathcal{Y}\rightarrow\mathbb{R}$ is the discriminator, i.e. a function parameterized by $\theta$ (whose exact form is specified by the user, it can be a neural network), and typically $\sigma(x)=\sfrac{1}{(1+e^{-x})}$. Maximizing this estimator of MI is equivalent to minimizing cross entropy between the prediction $\sigma(T(x,y))$ and the labels of the binary classification, with $P_{X,Y}$ (resp. $P_X\otimes P_Y$) is the distribution of class $1$ (resp. $2$).

\subsection{Graph2Vec}\cite{narayanan2017graph2vec} combines those two ideas to produce its graph embedding. The joint probability $P_{XY}=P_{X|Y}P_Y$ is constructed by sampling a graph $g$ from the dataset, and then by sampling a label $x$ from the sequence generated by \WL\ from this graph.
Minimizing the cross entropy with respect to $\theta$ leads to the following expression for the loss:  
\begin{align}
\mathcal{L}&=\mathbb{E}_{(x,g)\sim P_{XY}}\log{\sigma{(\theta_{g}\cdot\theta_{x})}}+\mathbb{E}_{(x,g)\sim P_X\otimes P_Y}\log{\sigma{(-\theta_{g}\cdot\theta_{x})}}
\end{align}
The discriminator function $T_{\theta}(x,y)=\theta_{y}\cdot\theta_{x}$ is taken as a dot product between a graph embedding $\theta_{g}\in\mathbb{R}^d$ and a \WL\text{ }label embedding $\theta_{x}\in\mathbb{R}^d$, which are vector of parameters randomly initialized and optimized with SGD. There is one such vector for each graph $g$ and each label $x$ produced by \WL\text{ }. The resulting graph embedding is $\mathcal{E}(g)=\theta_{g}$ while $\theta_{x}$ can be discarded. Optimizing this objective ensures to maximize the mutual information between the \WL\text{ }labels and the graph embeddings, which is a way to compress information about the distribution of \WL\ labels into the embedding.

\section{Contribution: Hierarchical Graph2Vec (HG2V)}
\label{s:contribution}
%We summarize our contributions as follows:
Our motivations for this work can be summarized as follows:
\begin{itemize}
    \item we first show that \WL\ fails to capture global scale information, which is hurtful for many tasks;
    \item we then show that such flaw can be corrected by the use of graph coarsening. In particular, Loukas' coarsening exhibits good properties in this regard;
    \item we finally show that the advantage of GNN over \WL\ is to be continuous functions in node features. They are robust to small perturbations.
\end{itemize}
Based on those observations, we propose a new algorithm building on graph coarsening and mutual information maximization, which we term \textbf{Hierarchical Graph2Vec} (or \textbf{HG2V}). It has the following properties:
\begin{itemize}
    \item The training is \textbf{unsupervised}. No label is required. The representation can be used for different tasks.
    \item The model is \textbf{inductive}, trained once for all with the graphs of the dataset in linear time. The training dataset is used as a prior to embed new graphs, whatever their underlying distribution.
    %coming from the same distribution or not.
    \item It handles \textbf{continuous nodes attributes} by replacing the hash function in \WL\ procedure by a Convolutional Graph NN. It can be combined with other learning layers, serving as pre-processing step for feature extraction.
    \item The model is \textbf{end-to-end differentiable}. Its input and its output can be connected to other deep neural networks to be used as building block in a full pipeline. The signal of the loss can be back-propagated through the model to train feature extractors, or to retrain the model in transfer learning. For example, if the node features are rich and complex (images, audio), a CNN can be connected to the input to improve the quality of representation.   
    \item The structures of the graph at all scales are summarized using Loukas coarsening. The embedding combines \textbf{local view and global view} of the graph. 
\end{itemize}
The resulting algorithm shares a similar spirit with Graph2Vec (MI maximization between node and graph descriptors), but it corrects some of its above-mentioned flaws. A high level overview of the method is provided in Algorithm~\ref{algo:HG2V_overview}. Table~\ref{tab:advantages} summarizes key properties of the method, against the other ones found in literature.  
  
In the following, we introduce  in Section~\ref{ssec:Loukas} the Loukas' coarsening method of \cite{loukas2018graph}, and detail how we use it in Section~\ref{ssec:hierarchyofneighborhoods}. Then, section~\ref{ssec:truncatedKrylov} deals with the continuity property achieved by GNN, while Section~\ref{ssec:HierarchicalNegSampling} explains how to train our proposed model \textbf{HG2V}.  

\begin{algorithm*}[t]
\caption{High-level version of the \textbf{HG2V} algorithm}\label{algo:HG2V_overview}
    \SetAlgoLined
    \KwResult{Graph embedding $\mathcal{E}(g)$ for each graph $g$}
    \KwIn{A training set of attributed graphs $g$, subset of $\mathbb{G}$ and the number of stages $L$ ; GNNs $F^l_{\theta}$ and $H^l_{\theta}$ with randomly initialized  $\theta$, $1\leq l\leq L$}
    \ForEach{batch of attributed graphs} {
        \ForEach{graph $g$ in the batch}{
             See Sec.~\ref{ssec:Loukas}: run Loukas' algorithm on $g$ to produce a sequence of
                   coarsened graphs $g^l$, $1\leq l\leq L$\;
        }
        \ForEach{level $1\leq l\leq L$}{
            \ForEach{graph $g^{l}$ in the batch}{
                \ForEach{node $u$ in $g^l$}{ See ~\ref{ssec:hierarchyofneighborhoods}:
                     Generate local neighborhood embedding $x^l(u)$ using $H^l_{\theta}$\;
                     Let $\mathcal{P}(u)\in g^{l+1}$ the image of $u$ after pooling a node of the coarsened graph\;
                     Generate node embedding $g^{l+1}(\mathcal{P}(u))$ using $F^l_{\theta}$\;
                     Create positive example $(x^l(u), g^{l+1}(\mathcal{P}(u))$\;
                }
            }
            \ForEach{pair of graphs $(g,g')$}{
                \ForEach{pair of nodes $(u,v)\in(g,\tilde{g})$}{
                     Create negative examples $(x^l(u), \tilde{g}^{l+1}(\mathcal{P}(v))$\;
                }
            }
              Minimize the cross entropy in Eq.~(\ref{eq:cross_entropy}) between positive and negative examples with discriminator $T(x,y)=x\cdot y$ and using Sec.~\ref{ssec:HierarchicalNegSampling}.
        }
    }
 %   \smallskip
\end{algorithm*}

\subsection{Loukas's Coarsening}
\label{ssec:Loukas}

In this section we detail the main drawback of \WL\text{ } procedure, and the benefit of graph coarsening to overcome this issue. For simplification, we will put aside the node attributes for a moment, and only focus on the graph structure. Even in this simplified setting, \WL\text{ } appears to be sensitive to structural noise.  

\subsubsection{Wesfeiler-Lehman Sensibility to Structural Noise}

The ability of \WL\text{ }to discriminate all graph patterns comes with the incapacity to recognize as similar a graph and its noisy counterpart. Each edge added or removed can strongly perturb the histogram of labels produced by \WL. Said otherwise, \WL\text{ } is not a good solution to inexact graph matching problem.
  
We perform experiments to evaluate the effect of adding or removing edges on different graphs. We randomly generate $100$ graphs of $500$ nodes each, that belong to four categories (cycle, tree, wheel and ladder), using the routines of NetworkX \cite{networkx} library. For each generated graph $g$, we randomly remove from $1$ to $10$ edges, sampled with independent and uniform probability, to create the graph $g'$. One may hope that such little modification over this huge edge set would not perturb excessively the labels of \WL\text{ } procedure.  
  
To evaluate the consequences of edge removal we propose to use as similarity score the \textit{intersection over union} of histogram of labels of $g$ and $g'$ at each stage $1\leq l\leq 5$:
\begin{equation}
\mathcal{S}^l(g,g')=100\times\frac{|\text{histo}(\WL^l{(g)})\cap\text{histo}(\WL^l{(g')})|}{|\text{histo}(\WL^l{(g)})\cup\text{histo}(\WL^l{(g')})|}
\end{equation}
The average similarity score $\tilde{\mathcal{S}}^l(g,g')$ over the $100$ graphs is reported in Figure \ref{fig:wlsensivity}.
  
The similarity decreases monotonically with the number of edges removed, even when restricting the procedure to the first stage (neighborhood of width $1$). On higher stages (wider neighborhood) the effect is even worse. On graphs with small diameter (such as \textit{wheel} graph or $3$-regular tree) a significant drop in similarity can be noticed. On ladder graph and cycle, sparse graphs with huge diameter, the effect of edge removal remains significant.  

\begin{figure}[t]
    \centering
     \begin{subfigure}[b]{0.48\linewidth}
    \includegraphics[width=1.\linewidth]{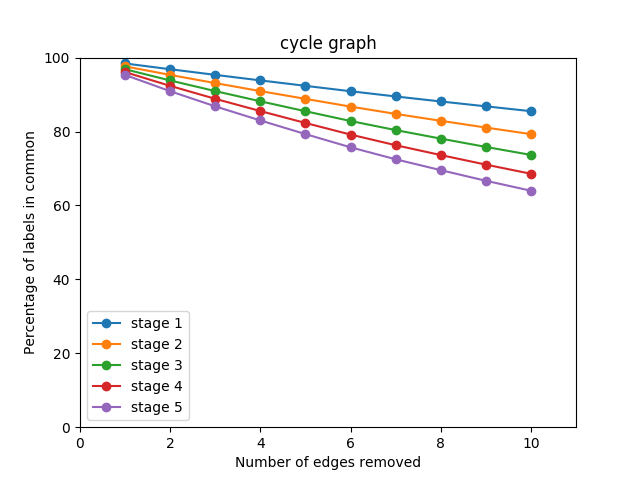} 
    \caption{Cycle: the $2$-regular graph with one giant component.}
    \vspace{1ex}
  \end{subfigure} 
  \begin{subfigure}[b]{0.48\linewidth}
    \includegraphics[width=1.\linewidth]{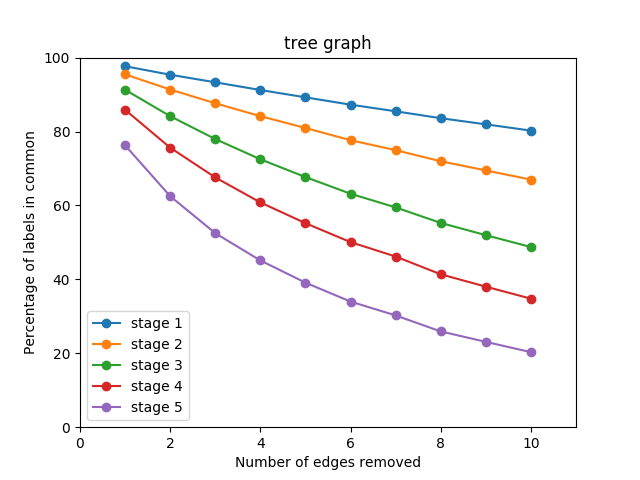} 
    \caption{Tree: $3$ children per node, except the ones at the last level.}
    \vspace{1ex}
  \end{subfigure} 
  \begin{subfigure}[b]{0.48\linewidth}
    \includegraphics[width=1.\linewidth]{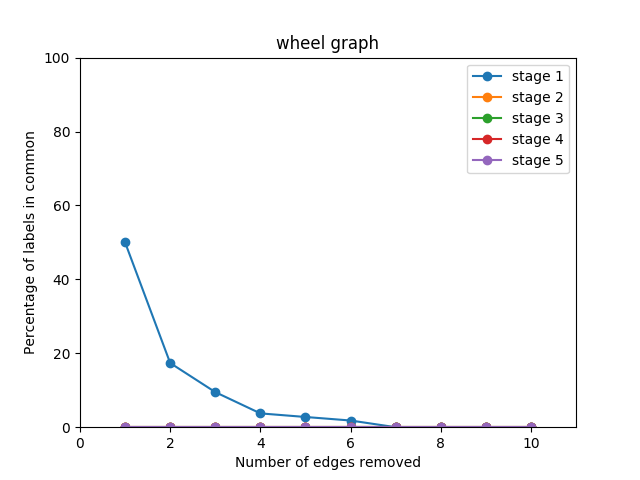} 
    \caption{Wheel: like cycle graph, with an additional node connected to all the others.}
    %\vspace{4ex}
  \end{subfigure}
  \hfill
  \begin{subfigure}[b]{0.48\linewidth}
    \includegraphics[width=1.\linewidth]{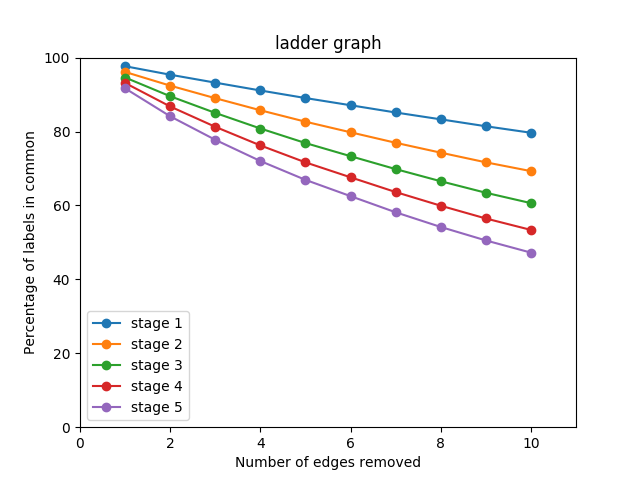} 
    \caption{Ladder: two paths of $250$ nodes each, where each pair of nodes is joined by an edge.}
    %\vspace{4ex}
  \end{subfigure} 
    \caption{Similarity score as a function of edges removed for different stages of \WL\text{ } iterations. The similarity score reaches $100\%$ for identical sets of labels, and $0\%$ for disjoint sets of labels.}
    \label{fig:wlsensivity}
\end{figure}
  
\subsubsection{Robustness to Structural Noise with Loukas's Coarsening}  
\WL\ procedure characterizes a graph as a sequence of rooted subtrees with increasing width. While this description is suitable for small patterns and inexact graph matching at local scale, it is very sensitive to structural noise. Adding or removing few edges (without hurting the global shape) changes the labels completely for subtrees of higher width. So as to characterize global features (e.g. communities, bridges, sparsest cuts...), we rely on a sequence of coarsened graphs based on the Loukas' procedure, that replaces the different iterations of neighborhoods of \WL.

\begin{Definition}[Graph Coarsening]
Graph coarsening is an operation mapping a graph $(V_1,E_1)$ to a new graph $(V_2,E_2)$ verifying $|V_2|<|V_1|$, using a \textbf{surjective graph homomorphism} $\mathcal{P}:V_1\rightarrow V_2$ called \textbf{pooling function}:
\begin{equation}
\label{eq:pooling}
\begin{split}
    (u,v)\in E_1&\implies (\mathcal{P}(u),\mathcal{P}(v))\in E_2\text{ or }\mathcal{P}(u)=\mathcal{P}(v)\\
    (\mathcal{P}(u),\mathcal{P}(v))\notin E_2&\implies (u,v)\notin E_1
\end{split}
\end{equation}
\end{Definition}

Loukas's method \cite{loukas2018graph} is a graph coarsening operation. This spectral reduction of attributed graphs offers the guarantee, in the spectral domain, that a coarsened graph approximates well its larger version. The idea is that the graph spectrum (eigenspaces and eigenvalues of their Laplacian) describes global and local structures of graphs. Structural noise may have very little consequences on the spectrum, depending of the nature of the intervention. Hence Loukas' coarsening, which preserves components associated with a low frequency in the graph spectrum, will produce a smaller graph with the same \say{global shape} as the input one, as demonstrated in Figure \ref{fig:coarsen_mnist}.  
  
In a nutshell, the method computes a sequence of projection matrices $(P_l)_{1\leq l\leq L}$ such that the Laplacian of the coarsened graph can be written as $L_{l+1}=P_l^{\mp}L_{l}P_l$ (where $\mp$ denotes the transposed pseudo-inverse). Since the algorithm is quite evolved and its explanation beyond the scope of our work, we refer to the original paper \cite{loukas2018graph} for an explicit description of Loukas' method and an extensive list of properties and guarantees.

\begin{figure}[b]
    \centering
    \includegraphics[width=11cm,height=29cm,keepaspectratio]{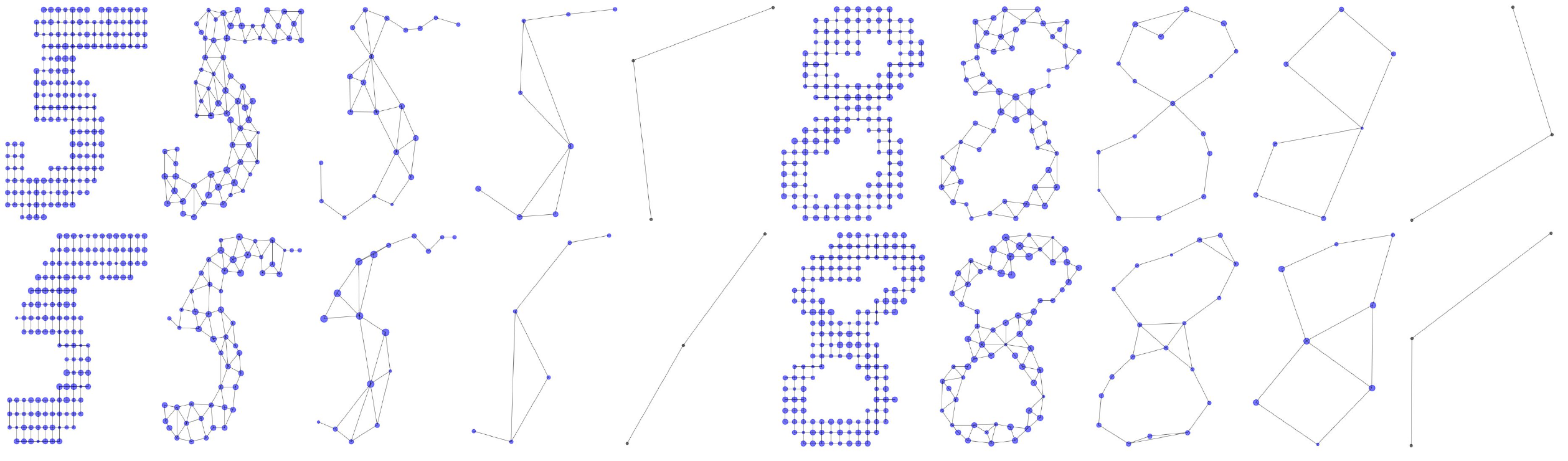}
    \caption{Coarsening of four graphs built from MNIST digits using Loukas' algorithm. The bigger the node, the wider the neighborhood pooled. Similar digits share similar shapes. }
    \label{fig:coarsen_mnist}
\end{figure}
\begin{figure}[t]
    \centering
     \begin{subfigure}[b]{0.48\linewidth}
    \includegraphics[width=1.\linewidth]{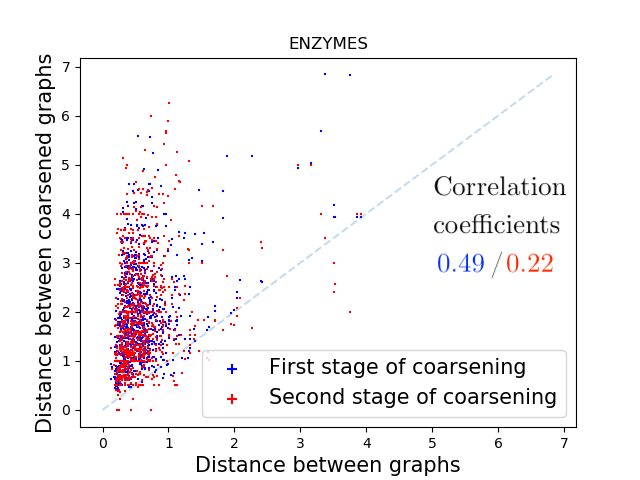} 

    %\vspace{1ex}
  \end{subfigure} 
  \begin{subfigure}[b]{0.48\linewidth}
    \includegraphics[width=1.\linewidth]{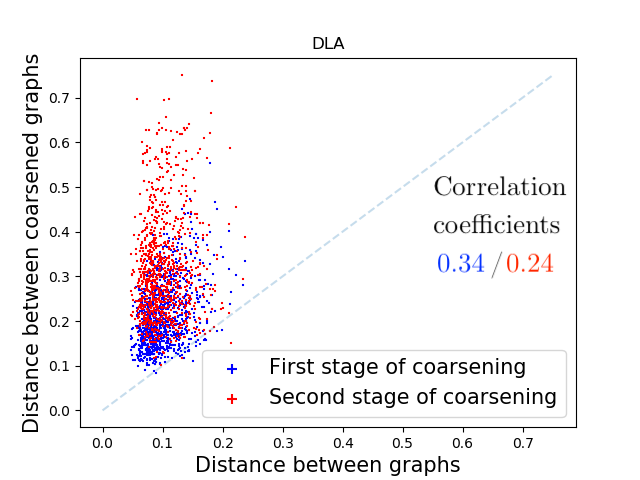} 

    %\vspace{1ex}
  \end{subfigure} 
  \begin{subfigure}[b]{0.48\linewidth}
    \includegraphics[width=1.\linewidth]{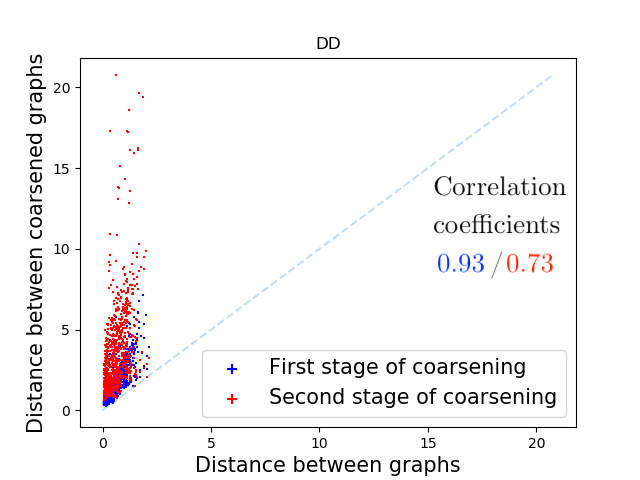} 

    %\vspace{4ex}
  \end{subfigure}
  \hfill
  \begin{subfigure}[b]{0.48\linewidth}
    \includegraphics[width=1.\linewidth]{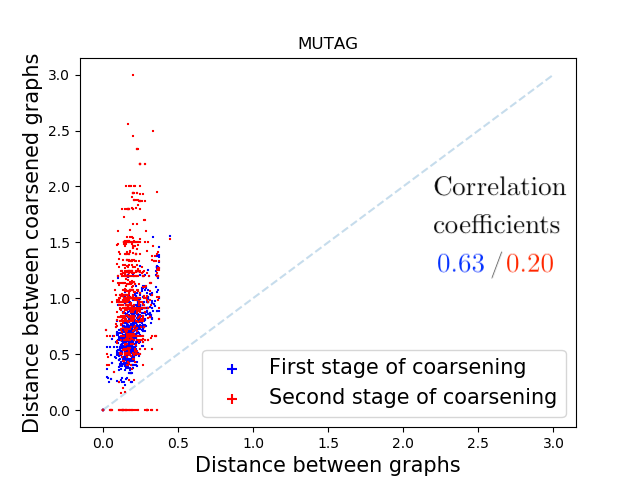} 

    %\vspace{4ex}
    
  \end{subfigure} 
    \caption{Wasserstein distance between spectra of graphs $g^0$ and $h^{0}$ sampled from different datasets (see Section~\ref{s:evalution}  for their description) compared to the same distance between their coarsened graphs $g^{1}$ and $h^{1}$ and their two times coarsened graphs $g^{2}$ and $h^{2}$ . In blue (resp. red) are given the correlation coefficients of the distance between $g^0$ \& $h^{0}$ and $g^1$ \& $h^{1}$ (resp. $g^2$ \& $h^{2}$). These coefficients have been computed after averaging 10 runs where we sampled 1000 graphs couples ($g^{0}$, $h^{0}$) for each dataset. As expected the more we coarsened graph the more the correlation coefficient decreases because coarsening always lost some  structural informations.  }
    \label{fig:distcoars}
\end{figure}
The interest of this coarsening method is that, if two graphs $g^{l-1}$ and $h^{l-1}$ are close enough, the coarsened graphs $g^l$ is itself a satisfying coarsening for $h^{l-1}$. By symmetry, the same result follows for $h^l$ and $g^{l-1}$. Hence, one may hope that $g^l$ and $h^l$ share similar patterns (Figure \ref{fig:distcoars}), and it will be advantageous for the \WL\ procedure.  
  
Such intuition can be confirmed experimentally. On four datasets (see Section~\ref{s:evalution} for their description), we sample two graphs $g^0$ and $h^0$ with independent and uniform probability.  
  
We measure their similarity using the Wasserstein distance between their spectra. Such choice is motivated by the fact that two graphs with different number of nodes have spectrum of different sizes. Wasserstein provides an elegant way to measure distance between sets of different size, in addition to be fast and easy to compute. 
  
\begin{Definition}[Wasserstein Distance between Graph Spectra]
Let $\bm{\lambda}=\{\lambda_1,\lambda_2,\ldots,\lambda_n\}$ and $\bm{\mu}=\{\mu_1,\mu_2,\ldots,\mu_m\}$ the spectra of two graphs. The Wasserstein distance is defined as:
\begin{equation}
    d_{\mathcal{W}}(\bm{\lambda,\mu})=\min_{\pi\in\Pi}\sum_{i,j}\pi(i,j)|\lambda_i-\mu_j|
\end{equation}
where $\Pi$ is the collection of measures over $\bm{\lambda}\times\bm{\mu}$ whose marginals are discrete uniform distributions over $\bm{\lambda}$ and $\bm{\mu}$ respectively.  
\end{Definition}
  
Since Loukas's coarsening preserves spectrum, we expect the distance between $g^0$ and $h^0$ to be correlated with the distance between their coarsened counterpart $g^1$ and $h^1$. Each dot in Figure~\ref{fig:distcoars} corresponds to an experiment, which are repeated $1000$ times. Interestingly, this correlation strongly depends of the underlying dataset.  
  
\subsection{Hierarchy of neighborhoods}
\label{ssec:hierarchyofneighborhoods}
%It induces a hierarchy of neighborhoods on which to apply \WL. Remember that \WL\ creates node embedding $x^l(u)$ that only depends of its neighborhood of radius $r$: the nodes at distance at most $r$ of $u$. Let us now consider $\mathcal{N}^r(u)$ the attributed subgraph induced by those nodes. We notice that $\mathcal{N}^r(u)\subset\mathcal{N}^{r+1}(u)$. The node embeddings are naturally organized into a hierarchy: each node $u$ and level $r$ characterizes a special neighborhood. The whole graph itself is just a very large neighborhood: for $g\in\mathbb{G}$ of diameter $D$ we have $g=\mathcal{N}^D(u)$.

Taking advantage of the previous observation, we propose to build a hierarchy of coarsened graphs $g_i^l$ using Loukas' method. It induces a hierarchy of nested neighborhoods $u, \mathcal{P}(u), \mathcal{P}(\mathcal{P}(u)), \dots, \mathcal{P}^L(u)$ by pooling the nodes at each level.  
  
We learn the \textbf{node embedding} $g^l(u)$ (of node $u$) at each level. This node embedding is used to produce a \textbf{local neighborhood embedding} $x^l(u)$ using function $H^l_{\theta}$, and to produce the node embedding of the next level $g^{l+1}(\mathcal{P}(u))$ using function $F^l_{\theta}$. Formally, the recurrent equations defining the successive embeddings are: 
\begin{align}
g^0(u)&=Z(u)\\
x^l(u)&=H^l_{\theta}(\{g^l(v) | v\in\mathcal{N}(u)\cup\{u\}\})\\
g^{l+1}(\mathcal{P}(u))&=F^l_{\theta}(\{g^l(v) | v\in\mathcal{N}(u)\cup\{u\}\})
\end{align}
The procedure is illustrated in Figure~\ref{fig:hierarchical}. In practice, functions $H^l_{\theta}$ and $F^l_{\theta}$ are graph neural networks parametrized by $\theta$, and whose exact form will be specified in next Section. 
\begin{figure}[t]
    \centering
    \includegraphics[width=11cm,height=27cm,keepaspectratio]{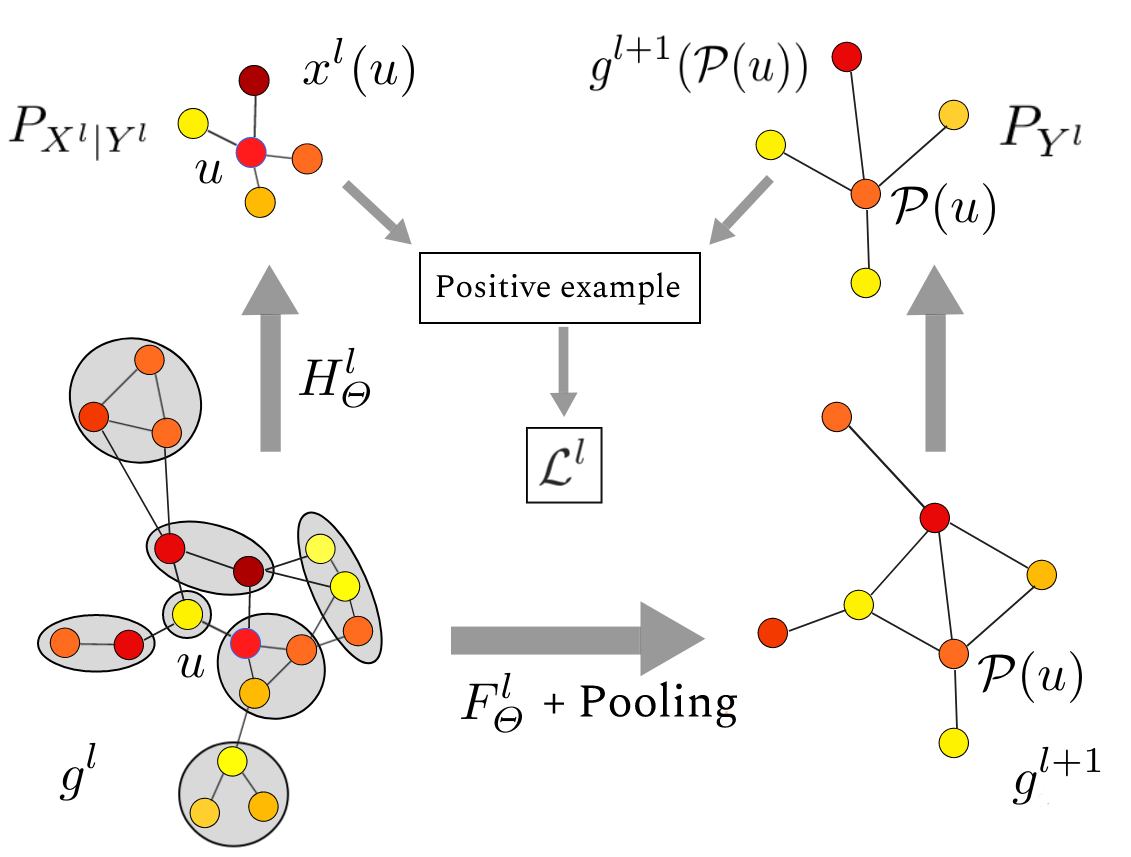}
    \caption{Single level of the pyramid. Local information $x^l(u)$ centered around node $u$ is extracted from graph $g^l$. The graph is coarsened to form a new graph $g^{l+1}$. There, $g^{l+1}(\mathcal{P}(u))$ captures information at a larger scale, centered on node $\mathcal{P}(u)$. The pair $(x^l(u),g^{l+1}(\mathcal{P}(u)))$ is used as positive example in the negative sampling algorithm, and it helps to maximize mutual information between global and local view.}
    \label{fig:hierarchical}
\end{figure}

\subsection{Handling continuous node attributes with Truncated Krylov} \label{ssec:truncatedKrylov}

The \WL\ algorithm uses a discrete hash function in Eq.~(\ref{eq:WLhashing}), with the issue that nodes sharing similar but not quite identical neighborhoods are considered different. If the differences are caused by noise in computations or measures, they should not result in much differences in the labels. For that, we relax the injectivity property of \WL\ by replacing it by a function with a continuity property. 

It is known that Graph Neural Networks (GNN) have a discriminative power at least equivalent to \WL\ \cite{maron2019provably,morris2018weisfeiler,xu2018powerful}. We require the opposite, and we emphasize the importance of not having a too strong discriminator. We use the extension of the Gromov-Wasserstein distance to attributed graphs, that requires the mapping to preserve both edge weights and node attributes. The resulting distance is a special case of the Fused Gromov-Wasserstein distance \cite{vayer:hal-02174322}.  
  
\begin{Definition}[Fused Gromov-Wasserstein distance.]
Let $g_1=(V_1,A_1,Z_1)$ and $g_2=(V_2,A_2,Z_2)$ be two attributed graphs in $\mathbb{G}$.
\begin{equation} \label{eq:gromov_wasserstein}
    d_\mathbb{G}(g_1,g_2)=\min_{\pi\in\Pi}\sum_{u,u',v,v'}\pi(u,v)\pi(u',v')\big(|A_1(u,u')-A_2(v,v')|+|Z_1(u)-Z_2(v)|+|Z_1(u')-Z_2(v')|\big)
\end{equation}
where $\Pi$ is the collection of measures over $V_1\times V_2$ whose marginals are discrete uniform distributions over $V_1$ and $V_2$.   
\end{Definition}
This definition allows us to state the continuity of GNN under this metric.  
\begin{Lemma}[GNN continuity.]
GNNs with continuous activation function are continuous on the topological space induced by the metric $d_\mathbb{G}$.\label{theorem:continuity}
\end{Lemma}

GNNs are usually parameterized functions learnable with stochastic gradient descent, and hence fulfill this continuity constraint.  Moreover, some attributes may be irrelevant, and should be discarded, or sometimes the graph is only characterized by some function (for example a linear combination) of the attributes. Such a function can be inferred directly from the dataset itself, simply by looking by the co-occurrences of some similar sub-graphs. Hence, the ability of GNN to learn features relevant for some specific task is an other advantage over \WL.  

GCN \cite{kipf2017semi} is as a baseline among this family of networks (see Section~\ref{s:related-work} for a brief survey). Unfortunately, \cite{luan2019break} have shown that they behave badly when stacked. They propose Truncated Krylov GNN in replacement of GCN, correcting its flaws, with better theoretical guarantees and empirical evidence of their superiority over GCN. We chose it, encouraged by the promising results of early experiments. It amounts to consider, for a given node, a receptive field extended to the neighborhood of nodes at most at given distance $a$. The resulting layer uses the normalized adjacency matrix $\tilde{A}=D^{-\sfrac{1}{2}}(A+I)D^{-\sfrac{1}{2}}$ (like in GCN) of the graph $g^l$:
\begin{equation}
\label{eq:TruncatedKruylov}
\begin{split}
g^{l+1}&=\text{Pool}(\tanh{([g^l,\tilde{A}g^l,\tilde{A}^2g^l,\dots,\tilde{A}^ag^l]\theta^l_1)})\theta^l_2+\theta^l_3 \\
x^l&=\tanh{([g^l,\tilde{A}g^l,\tilde{A}^2g^l,\dots,\tilde{A}^ag^l]\theta^l_4)}.
\end{split}
\end{equation}
The vector $\theta$ is made of the trainable parameters, and $[\ldots]$ denotes the concatenation operation. We abused the notations, using $x$ and $g$ as real-valued vectors indexed by node set $V$. In Eq.~(\ref{eq:TruncatedKruylov}), one implements a specific graph filter of order $a$, close to the ones considered by Chebyshev polynomial approximation in \cite{defferrard2016convolutional}.\\
\textbf{Pool} denotes the merging of nodes. Their features are summed because it preserves information on multiset, as argued in \cite{xu2018powerful}.  
  
Note that $F^l_{\theta}$ and $H^l_{\theta}$ span the same neighborhood (up to distance $a$). But the usage of coarsening constraints $F^l_{\theta}$ to summarize the information content in the group of pooled nodes. It is this last point that makes $g^{l+1}$ a global descriptor. The affine transformation $X\mapsto X\theta^l_2+\theta^l_3$ in~(\ref{eq:TruncatedKruylov}) is used to improve the expressiveness of the model thanks to a bias term. Indeed, $g^{l}$ must both summarize the current level $l$ an be used as input for the next one, while $x^l$ only aims to summarize the neighborhood.
  
\subsection{Hierarchical Negative Sampling}  \label{ssec:HierarchicalNegSampling}

Following the principles of Graph2Vec we aim to \textbf{maximize mutual information} between \textbf{node embeddings} and \textbf{local neighborhood embeddings} at each level. We achieve this by training a classifier to distinguish between positive and negative examples of such pairs of embeddings. Positive pairs are defined as node $v$ in level $l$ and its ``super node'' $\mathcal{P}(v)$ in level $l+1$.  
  
More precisely, consider the $l$-th level of the hierarchical pyramid. The probability $P_{Y^l}$ is built by sampling a graph $g$ from the dataset, then by sampling a node $u$ with uniform probability from this graph. $P_{X^l|Y^l}$ is obtained by sampling a node from $\mathcal{P}^{-1}(u)$. It gives a pair $(x^l(\mathcal{P}^{-1}(u)), g^{l+1}(u))$ sampled according to $P_{X^lY^l}=P_{Y^l}P_{X^l|Y^l}$. The negative pairs are built as usual from the independent probability $P_{X^l}\otimes P_{Y^l}$. The corresponding loss function takes the form:
\begin{equation}
\mathcal{L}^l=\mathbb{E}_{(x^l_j(u),g^l(\mathcal{P}(u)))\sim P_{X^lY^l}}\log{\sigma{(x^l(u)\cdot g^l(\mathcal{P}(u)))}}+\mathbb{E}_{(x^l(u),g^l(v))\sim P_{X^l}\otimes P_{Y^l}}\log{\sigma{(-x^l(u)\cdot g^l(v))}}
\end{equation}
  
The overall method is described in Algorithm~\ref{algo:HG2V_overview}.
  
A descriptor $\mathcal{E}_l(g)$ for each level $l$ of the graph $g$ is computed by global pooling over node embeddings. The final graph embedding is obtained by concatenating those descriptors. The sum aggregator is always preferred over mean aggregator, because of its better ability to distinguish between multisets instead of distributions, as argued in \cite{xu2018powerful}. Hence, the learned representation is:
\begin{align}
\mathcal{E}_l(g) &= \left[\sum_{u\in V}g^l(u),\max_{u\in V}g^l(u) \right]\\
\mathcal{E}(g) &=\text{CONCATENATE}\left(\mathcal{E}_1(g), \mathcal{E}_2(g), ... , \mathcal{E}_L(g)\right)
\end{align}

In the Loukas' method, the number of stages is not pre-determined in advance: only the final coarsened graph size can be chosen. When the number of stages produced by the Loukas method is not aligned with the depth of the neural network, the pooling is either cropped (too much stages) or completed with identity poolings (not enough stages). The resulting vector can be used for visualization purposes (PCA, t-SNE) or directly plug into a classifier with labels on graphs.

\subsubsection{Complexity}

\noindent \textbf{Time.} The coarsening of each graph can be pre-computed once for all, in linear time in the number of nodes and edges \cite{loukas2018graph}. Hence, the main bottleneck is the maximization of mutual information. Let $|V|$ the maximum number of nodes in a graph. Let $B$ the batch size. Let $L$ the number of stages. Let $a$ the order of Truncated Krylov. Let $d$ the dimension of node embedding. The complexity of the algorithm \textit{per batch} is $\mathcal{O}(BLa|V|^3+BLa|V|^2d+BLa|V|d^2+B^2Ld)$. The first term is due to exponentiation of adjacency matrix, the second one to diffusion along edges, the third one to forward pass in network layers, and the last one to the Cartesian product to create positive and negative examples. The average complexity \textit{per graph} is hence: $\mathcal{O}(La|V|(|V|+d)^2+BLd)$. The most sensitive factor is the number of nodes, followed by the the number of features and the batch size. The magnitude of those matrices allows to handle graphs with hundred of nodes efficiently on modern GPUs, with embedding as big as $512$ and up to $8$ graphs per batch and $5$ stages of coarsening. In practice, the bottleneck turn out to be the pre-computation of coarsening, which do not benefit of GPU speed up.  

\noindent \textbf{Space.} Note that for datasets with small graphs we have $\text{dim}(\mathcal{E}_i)=2Ld>\text{dim}(X_i)+\text{dim}(A_i)=n|V|+|V|^2$. However, when the number of nodes exceeds $50$, the embedding size is always smaller than the adjacency matrix. Hence, this method is more suitable for big graphs with hundreds of nodes.

\section{Evaluation}
\label{s:evalution}

The code for the proposed method and to reproduce the experiments can be found on: \url{https://github.com/Algue-Rythme/GAT-Skip-Gram}.\\
%The code of the experiments can be found on: \url{https://github.com/anonymized}.\\
An experimental evaluation is conducted on various attributed graphs datasets, and the quality of the embeddings is assessed on supervised classification tasks in Section~\ref{s:supervised}. Our method is inductive: the model can be trained over a dataset and be used to embed graphs coming from another dataset. This property is analysed in Section~\ref{s:inductive}.
\textbf{HG2V} differs from Graph2Vec by the usage of GNN and Loukas coarsening. The influence of those two elements is analysed with ablatives studies in Section~\ref{s:ablative}.

\begin{table*}[t]
%\centering
%\raggedright
\hspace*{-1.cm}
\scalebox{0.8}{
\begin{tabular}{|M{2cm}M{0.8cm}M{0.8cm}|M{1.8cm}M{1.8cm}M{1.8cm}M{1.8cm}M{1.8cm}M{1.8cm}|M{1.8cm}M{1.8cm}|}
\hline
\textbf{DATASET} & \#graphs & \#nodes & \textbf{HG2V} (ours) & \textbf{Graph2Vec}& \textbf{Infograph} & \textbf{DiffPool} (supervised) & \textbf{GIN} (supervised)& \textbf{MinCutPool} (supervised) & \textbf{WL-OA} (kernel) & \textbf{WWL} (kernel)\\
\hline
\hline
IMDB-m & 1500 & 13 & $47.9\pm1.0$ & $\bm{50.4\pm0.9}$ & $49.6\pm0.5$ & $45.6\pm3.4$ & $48.5\pm3.3$ & \xmark & \xmark & \xmark\\
\hline
PTC\_FR & 351 & 15 & $\bm{67.5\pm0.5}$ & $60.2\pm6.9$ & \xmark & \xmark & \xmark & \xmark & $\bm{63.6\pm1.5}$ & \xmark\\
\hline
FRANK. & 4337 & 17 & \bm{$65.3\pm0.7$} & $60.4\pm1.3$ & \xmark & \xmark & \xmark & \xmark & \xmark & \xmark\\
\hline
MUTAG & 188 & 18 & $81.8\pm1.8$ & $83.1\pm9.2$ & \bm{$89.0\pm1.1$} & \xmark & \xmark & \xmark & $84.5\pm1.7$ & $\bm{87.3\pm1.5}$\\
\hline
IMDB-b & 1000 & 20 & $71.3\pm0.8$ & $63.1\pm0.1$& \bm{$73.0\pm0.9$} & $68.4\pm3.3$ & $71.2\pm3.9$ & \xmark & \xmark & $\bm{74.4\pm0.8}$\\
\hline
NCI1 & 4110 & 30 & $76.3\pm0.8$ & $73.2\pm1.8$ & \xmark & $76.9\pm1.9$ & \bm{$80.0\pm1.4$} & \xmark & $\bm{86.1\pm0.2}$ & $85.8\pm0.2$\\
\hline
NCI109 & 4127 & 30 & $\bm{75.6\pm0.7}$ & $74.3\pm1.5$ & \xmark & \xmark & \xmark & \xmark & $\bm{86.3\pm0.2}$ & \xmark\\
\hline
ENZYMES & 600 & 33 & $\bm{66.0\pm2.5}$ & $51.8\pm1.8$ & \xmark & $59.5\pm5.6$ & $59.6\pm4.5$ & \xmark & $59.9\pm1.1$ & $\bm{73.3\pm0.9}$\\
\hline
PROTEINS & 1113 & 39 & $75.7\pm0.7$ & $73.3\pm2.0$ & \xmark & $73.7\pm3.5$ & $73.3\pm4.0$ & $\bm{76.5\pm2.6}$ & $76.4\pm0.4$ & $\bm{77.9\pm0.8}$\\
\hline
MNIST & 10000 & 151 & $\bm{96.1\pm0.2}$ & $56.3\pm 0.7$ & \xmark & \xmark & \xmark & \xmark & \xmark & \xmark\\
\hline
D\&D & 1178 & 284 & $79.2\pm0.8$ & $58.6\pm0.1$ & \xmark & $75.0\pm3.5$ & $75.3\pm2.9$ & $\bm{80.8\pm2.3}$ & $79.2\pm0.4$ & $\bm{79.7\pm0.5}$\\
\hline
REDDIT-b & 2000 & 430 & $91.2\pm0.6$ & $75.7\pm1.0$ & $82.5\pm1.4$ & $87.8\pm2.5$ & $89.9\pm1.9$ & $\bm{91.4\pm1.5}$ & $89.3$ & \xmark\\
\hline
DLA & 1000 & 501 & $\bm{99.9\pm0.1}$ & $77.2\pm2.5$& \xmark & \xmark & \xmark & \xmark & \xmark & \xmark\\
\hline
REDDIT-5K & 4999 & 509 & $55.5\pm0.7$ & $47.9\pm0.3$ & $53.5\pm1.0$ & $53.8\pm1.4$ & $\bm{56.1\pm1.7}$ & \xmark & \xmark & \xmark\\
\hline
\end{tabular}
}
\smallskip
\caption{Accuracy on classification tasks. HG2V is trained over both \textit{TrainVal+Test} splits, without using labels due to its unsupervised nature. Model selection of C-SVM and hyper-parameters of HG2V have been done with 5-cross validation over \textit{TrainVal} split. We report on the accuracy over the \textit{Test} split, averaged over $10$ runs, and with standard deviation. Unavailable result marked as \xmark.}
\label{tab:results}
\end{table*}

\subsection{Datasets}
\noindent \textbf{Standard datasets.}
We use standard datasets from literature: PROTEINS, ENZYMES, D\&D, NCI1, NCI109, MUTAG, IMDB (binary and multi) and PTC\_FR downloaded from \cite{KKMMN2016}. We also use the challenging REDDIT (binary and 5K versions).  
  
\noindent \textbf{Synthetic datasets.} Additionally, we introduce a novel dataset for the community: Diffusion Limited Aggregation (DLA). The attributed graphs are created by a random process that makes the graphs scale free, DLA being known to be fractal objects. The graphs have then an interesting property that justifies the creation of a new benchmark. In addition it provides natural features for the nodes: the coordinates in space (and this is why we prefer it to the well known and usual scale-free network model that is the Barabasi-Albert model). We refer to Appendix~\ref{app:dla} for more details. The code can be found on \url{https://github.com/Algue-Rythme/DiffusionLimitedAgregation}
  
\noindent \textbf{Image datasets.} We convert MNIST and USPS popular datasets of the computer vision community into graphs, by removing blank pixels, and adding luminosity and $(x,y)$ coordinates as node features. To solve the task, the method should be able to recognize the shape of the graph, which is a global property. Due to the size of these datasets, we restrain ourselves to a subset of $10,000$ images randomly sampled (instead of the $70,000$ available).  
  
\noindent \textbf{Frankenstein dataset} was created in \cite{orsini2015} by replacing nodes labels of BURSI dataset with MNIST images.

\noindent \textbf{Pre-processing} All the continuous node attributes are centered and normalized. The discrete node labels use one hot encoding. When there is no node feature, we use the degree of the node instead. 

\subsection{Supervised classification} \label{s:supervised}

In the first task, the model is trained over all the attributed graphs in the dataset. The quality of these embeddings is assessed on a supervised classification task. The classifier used is C-SVM with RBF kernel from scikit-learn library.  

    \subsubsection{Training Procedure}
The embedding are trained over $10$ epochs. At each step, $8$ graphs are picked randomly from the dataset, and all the vocabulary from these graphs is used to produce positive and negative pairs (by cartesian product). Hence, the number of negative examples in each batch is higher than the number of positive examples. Consequently we normalize the loss of negative samples to reduce the unbalance. The optimizer is Adam \cite{kingma2014adam} and the learning rate follows a geometric decay, being divided by approximately $1000$ over the 10 epochs.

\subsubsection{Model selection} The relevant hyper-parameters of HG2V are the number of features $d\in\{16,128,256\}$ at each stage, the receptive field of Truncated Krylov $a\in\{2,4\}$, and the maximum depth of Loukas' coarsening $L\in\{3,5\}$. They are selected using a grid search. Five random split of the dataset are generated: \textit{TrainVal} (80\% of the data) and \textit{Test} (20\% of the data). HG2V is trained over \textit{TrainVal}. Then, a C-SVM is trained over \textit{TrainVal}, using 5-cross validation grid search for the selection of its hyper-parameters (C, Gamma). The average validation score of the best C-SVM classifier is used to select the hyper-parameters $(d^*,a^*,L^*)$ of HG2V. The average test score of the best C-SVM classifier (accuracy over \textit{Test} split) is reported in Table~\ref{tab:results}. Note that HG2V could be trained on the test set \textit{without} using labels, due to its unsupervised nature. We decided not doing it, to ensure fair comparison with the other methods.  

\begin{table}[!t]
 \centering
 %\raggedright
 %\hspace*{-0.4cm}
\begin{tabular}{|M{2cm}M{2cm}|M{2cm}|M{3cm}|}
\hline
\textbf{Training Set} & \textbf{Inference Set} & \textbf{Accuracy (Inference)} & \textbf{Delta with baseline (see Table~\ref{tab:results})}\\ 
\hline
\hline
MNIST & USPS  & $94.86$ & \xmark\\ 
\hline
USPS & MNIST  & $93.68$ & $-2.40$ \\
\hline
REDDIT-b & REDDIT-5K & $55.00$ & $-0.48$\\
\hline
REDDIT-5K & REDDIT-b & $91.00$ & $-0.15$\\
\hline
REDDIT-b & IMDB-b & $69.00$ & $-2.25$\\
\hline
REDDIT-5K & IMDB-b & $69.50$ & $-1.75$\\
\hline
MNIST & FASHION MNIST & $83.35$ & \xmark\\
\hline
\end{tabular}
\smallskip
\caption{Accuracy on classification tasks by training on some input distribution and performing inference on an other. The hyper-parameters selected are identical to the ones of Table~\ref{tab:results}.  
%\vspace*{-8mm}
}
\label{tab:transferresults}
\end{table}

\subsubsection{Baselines}

We compare our work to various baselines of the literature:

\noindent \textbf{Kernel Methods} All the results reported are extracted from the corresponding papers \cite{kriege2016valid,nikolentzos2019graph}, giving an idea of the best possible performance achievable by \WL-Optimal Assignment~\cite{kriege2016valid} and the Wasserstein Wesfeiler-Lehman~\cite{togninalli2019wasserstein} graph kernels. It almost always outperform inductive methods based on neural networks. However, like many kernel-based method, they have quadratic time complexity in the number of graphs, which is prohibitive for dealing with large datasets. Due to its very high scores and its sensibly different design, we consider them apart and we never highlight their results (even if the best).  

\noindent \textbf{DiffPool, GIN} We report the results of the rigorous benchmarks of \cite{Errica2020A}, including the popular DiffPool \cite{NIPS2018_7729} and GIN \cite{xu2018powerful}. Those algorithms are end-to-end differentiable, but they are \textbf{supervised}. DiffPool also relies on graph coarsening, but their pooling function is learned, while Loukas coarsening is task agnostic.  
  
\noindent \textbf{MinCutPool} We report the results of the original paper \cite{bianchi2020spectral}. Note that they used not only node degree, but also clustering coefficient as node features. Consequently they benefit from additional information compared to our setting. 

\noindent \textbf{Infograph} We also report the results of Infograph \cite{Sun2020InfoGraph}. It is the closest method to our work: it is unsupervised, end-to-end differentiable, and also relies on mutual information maximization, but it does not benefit of coarsening. Infograph is currently the state of the art in \textbf{unsupervised} graph representation learning.  

\subsubsection{Results}
We note that get substantial improvements over Graph2Vec baseline for many datasets, more specifically when the graph are big and carry high dimensional features.   
  
For FRANKENSTEIN, if we connect a randomly initialized 2-layer CNN to the input of the model for better feature extraction, the results are improved and reach $66.5\pm0.4\%$ which is a noticeable improvement. Thanks to the end-to-end differentiability of the model, the CNN is trained with backpropagation, benefiting from the unsupervised loss signal.  
  
On the notably difficult REDDIT-B and REDDIT-5K we reach high results, comparable to SOTA. The coarsening operation is beneficial to these datasets, considering the size of the graphs. On datasets with smaller graphs, the results are less significant.  
  
\subsubsection{Computation time.}
Training on only $1$ epoch already provides a strong baseline for molecule datasets, and lasts less than 1 minute, using GTX 1080 GPU. The most challenging dataset was REDDIT-MULTI-5K, trained on V100 GPU, with $5000$ graphs, an average of $508$ nodes and $595$ edges per graph. The pre-computation of Loukas's coarsening required 40 minutes (that can be improved with parallelization, not implemented). After this step, the runtime never exceed 190 seconds per epoch, for a total training time of 70 minutes.  

\subsection{Inductive learning} \label{s:inductive}

The method is inductive: it allows us to train \textbf{HG2V} on a dataset and test on another dataset.  The training set is used to extract relevant features, that are expected to be seen again during inference. In previous section~\ref{s:supervised} we showed that when the domain of training set and inference set are the same (different splits of the same dataset) it provides good results. In this section, we prove that even when the training set and the inference set are disjoint, the model is still able to produce good representations.  
  
We emphasize that this property is specific to unsupervised inductive methods. Hence, it is not possible to perform such experiment on Graph2Vec (transductive), or on models trained with supervised loss (e.g GIN, DiffPool). To the best of our knowledge, no such domain adaption experiments have been done in graph classification tasks. Consequently, we miss comparisons to other methods. A more in-depth analysis of this property is left as future work. Our preliminary results with this regard are summarized in Table~\ref{tab:transferresults}, and we hope that it will encourage community to perform similar experiments.\\

\noindent \textbf{Hyper-parameters} The best hyper-parameters found in the previous section are kept as is without further hyper-parameter tuning. The goal was not to reach the best possible result, but ensures that we can reuse trained weights to produce quickly good embeddings.

\noindent \textbf{USPS} is a dataset of handwritten digits similar to MNIST, hence we expect that a model trained on it can also embed MNIST graphs.
  
\noindent \textbf{FASHION MNIST} have been introduced in \cite{xiao2017/online} is built similarly to MNIST: a set of $28\times 28$ grayscale images. We only work on a subset of $10,000$ training examples.  

\begin{table}[!t]
 \centering
 %\raggedright
 %\hspace*{-1cm}
\begin{tabular}{|cc|cc|cc|cc|c|}
\hline
& & \multicolumn{2}{c|}{\textbf{HG2V}} & \multicolumn{2}{c|}{\textbf{Graph2Vec+GNN}} & \multicolumn{2}{c|}{\textbf{Graph2Vec+Loukas}} & \textbf{Graph2Vec}\\
\textbf{DATASET} & \#nodes & Accuracy & Delta & Accuracy & Delta & Accuracy & Delta &\\
\hline
\hline
\textbf{IMDB-b} & 20 & $70.85$ & $+7.75$ & $70.70$ & $+7.60$ & $57.5$ & $-5.60$ & $63.10$\\
\hline
\textbf{NCI1} & 30 & $77.97$ & $+4.75$ & $75.40$ & $+2.18$ & $65.45$ & $-7.77$ & $73.22$\\
\hline
\textbf{MNIST} & 151 & $95.83$ & $+39.56$ & $91.05$ & $+34.78$ & $72.5$ & $+16.23$ & $56.27$\\ 
\hline
\textbf{D\&D} & 284 & $78.01$ & $+19.37$ & $79.26$ & $+13.16$ & $66.10$ & $+7.45$ & $58.64$\\
\hline
\textbf{REDDIT-B} & 430 & $91.95$ & $+16.23$ & \textbf{OOM} & \xmark & $82.50$ & $+6.78$ & $75.72$\\ 
\hline
\end{tabular}
\smallskip
\caption{Ablative studies. The accuracy on test set is in column "Accuracy". The column "Delta" corresponds to the difference in average accuracy with Graph2Vec. \textbf{OOM} is Out of Memory Error.}
\label{tab:ablativeresults}
\end{table}
  
\subsubsection{Results} We see that the model can easily transfer from one domain to another. MNIST seems to be a better prior than USPS, which is a behavior previously observed in transfer learning. On the other datasets, the accuracy drops compared to the baseline: it is expected, since the new graphs are out of distribution of the training set. Nevertheless, even in this unfavorable setting, the accuracy remains within a comparable range to the baseline.  
  
We conclude that the inductivity property is not only theoretical, but also a property that can be checked in practice, without specific hyper-parameter tuning or costly manual adaptation.

\subsection{Ablative Studies} \label{s:ablative}

We perform ablative studies by removing separately Loukas coarsening and GNN. If we remove both, we fall back to Graph2Vec. The dimension of the embeddings is chosen equal to $1024$.  
  
\noindent \textbf{Graph2Vec+GNN} We remove Loukas Coarsening. The only difference with Graph2Vec is the replacement of \WL\ iterations with forward pass through a GNN. All the attributes available are used.\\
\noindent \textbf{Graph2Vec+Loukas} We remove GNN. The resulting algorithm is Graph2Vec applied on the sequence of coarsened graphs. On the coarsened graphs, new labels are generated by concatenating and hashing the labels of the nodes pooled (like a \WL\ iteration would do). The sequence of (unconnected) graphs is fed into Graph2Vec. Continuous attributes are ignored because \WL\ can not handle them.\\
The results are summarized in Table~\ref{tab:ablativeresults}.

\subsubsection{Results}
On datasets with small graphs (less than 30 nodes in average) the use of coarsening is hurtful, resulting in loss compared to Graph2Vec. As soon the graphs get big, coarsening leads to huge improvements in accuracy. We also notice that the usage of GNN and its ability to handle continuous attributes, and to be trained to extract co-occurring features, leads to significant improvements on all datasets.  
  
\begin{figure}[!t]
    \centering
    \includegraphics[width=13cm,height=29cm,keepaspectratio]{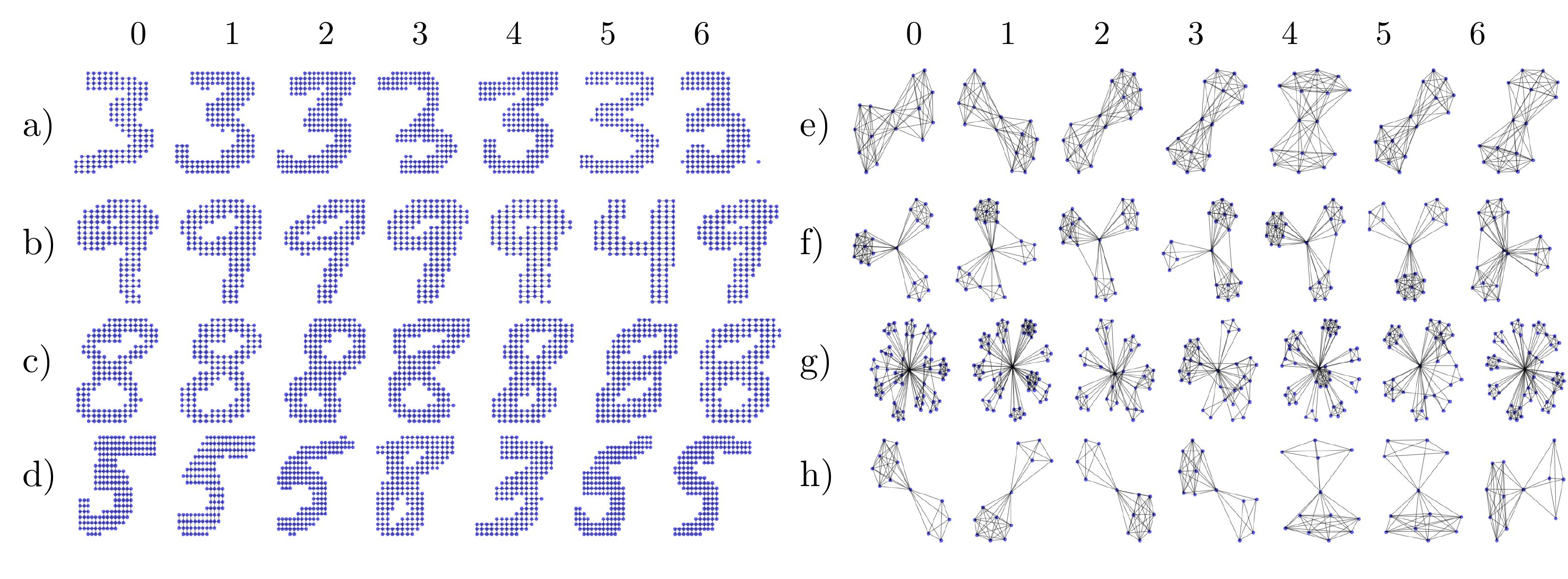}
    \caption{Six nearest neighbors of graphs in the learned latent space for four graphs from IMDB-b and MNIST. Column 0 corresponds to a randomly chosen graph, then the six nearest neighbors are drawn in increasing distance order from left to right (1 to 6).}
    \label{fig:nearest_neighbors}
\end{figure}

\subsection{Latent space}

Figure~\ref{fig:nearest_neighbors} illustrates  some graphs (leftmost column) with their six closest neighbors in the learned latent space (right columns 1 to 6), from the closest to the farthest, taken from MNIST and IMDB. We observe that isomorphic graphs share similar embeddings, and even when the graphs are not quite isomorphic they exhibit similar structures.  

\section{Conclusion}

We proposed a new method for Graph Representation Learning. It is fully unsupervised, and it learns to extract features by Mutual Information maximization. The usage of Loukas' coarsening allows us to tackle all scales simultaneously, thanks to its capacity to preserve the graph spectrum. The method is inductive, and can be trained on only a fraction of the data before being used in transfer learning settings. Despite being unsupervised, the method produces high quality embeddings leading to competitive results on classification tasks against supervised methods.

\appendix

\section{Proofs}

\subsection{Proof of lemma~\ref{theorem:continuity}}

\begin{Lemma}[GNN continuity.]
As $d_\mathbb{G}$ is a metric, then GNNs with continuous activation function are continuous under the induced topology.
\end{Lemma}
\begin{proof}
We use Message Passing Neural Networks (MPNN) framework  introduced in \cite{gilmer2017neural}. In order to demonstrate our theorem, we use the same notations as their paper. Let $g=(V,A,Z)\in\mathbb{G}$. Let $u\in V$ a node, and $t$ the index of a layer. Then:
\begin{align}
   h^0(u) &= Z(u)\\
   m^{t+1}(u) & = \sum\limits_{v\in\mathcal{N}(u)} M_t(h^t(u), h^t(v), e_{uv})\\
   h^{t+1}(u) & = U_t( h^t(u), m^{t+1}(u))
\end{align}

Every GNN layer fall under MPNN framework. When the activation function used is continuous, the functions $M_t$ and $U_t$ are also continuous, because they involve matrix products (which are continuous operations). Moreover, the sum over $\mathcal{N}(u)$ is finite because we consider finite graphs. Consequently, $h^{t+1}(u)$ depend continuously of $\{h^t(v), v\in\mathcal{N}(u)\}$ through $S_t=U_t\circ M_t$. By composition, the result holds for any number of layers $T$. Let $F$ the resulting function on graph domain $\mathbb{G}$, i.e if $(V',A',Z')=g'=F(g)$ then $A=A'$ and $V=V'$ (GNNs do not modify topology), and for every node $u\in V$ we have $Z'(u)=h^{T}(u)$.  
  
Let $(g_n)_{n\in\mathbb{N}}\in\mathbb{G}^{\mathbb{N}}$ a sequence of attributed graphs such that $\lim_{n\to\infty} g_n=g$ under the topology induced by $d_\mathbb{G}$. Let $(\pi_n:V_n\times V)_{n\in\mathbb{N}}$ the corresponding set of transportation plans between $g_n$ and $g$. Let $g'_n=(V'_n,A'_n,Z'_n)=F(g_n)$ and $g'=(V',A',Z')=F(g)$. Then:
\begin{equation}
\label{eq:gnnbound}
    \begin{aligned}
    d_\mathbb{G}(g'_n,g') &\leq \sum_{u,u',v,v'}\pi_n(u,v)\pi_n(u',v')\left(|A'_n(u,u')-A'(v,v')|+|Z'_n(u)-Z'(v)|+|Z'_n(u')-Z'(v')|\right)\\
    &=\sum_{u,u',v,v'}\pi_n(u,v)\pi_n(u',v')\left(|A_n(u,u')-A(v,v')|+|Z'_n(u)-Z'(v)|+|Z'_n(u')-Z'(v')|\right) 
\end{aligned}
\end{equation}
Now, because $\lim_{n\to\infty} g_n=g$ we have necessarily:
\begin{align}
    \lim_{n\to\infty} &\pi_n(u,v)\pi_n(u',v')|Z_n(u)-Z(u)| = 0\\
    \lim_{n\to\infty} &\pi_n(u,v)\pi_n(u',v')|A_n(u,u')-A(v,v')| = 0
\end{align}
The continuity of $S_t$ allows to further conclude that:
\begin{equation}
\lim_{n\to\infty} \pi_n(u,v)\pi_n(u',v')|Z'_n(u)-Z'(u)| = 0
\end{equation}
Finally, the right hand size of \eqref{eq:gnnbound} must have limit $0$, hence:
\begin{equation}
\lim_{n\to\infty} g'_n=g'
\end{equation}
We just proved that $\lim_{n\to\infty} g_n=g$ implies $\lim_{n\to\infty} g'_n=g'$, which is precisely the definition of $F$ being continuous w.r.t the topology induced by $d_\mathbb{G}$.  
\end{proof}

\section{Additional visualizations of the embeddings}

We present other randomly sampled graphs and their six closest neighbors from MNIST-Graph, IMDB and PTC Datasets in, respectively, Figures \ref{fig:mnist_neighbors}, \ref{fig:imdb_neighbors} and \ref{fig:ptc_neighbors}.

\begin{figure}%[h]
    \centering
    \includegraphics[width=9cm,height=29cm,keepaspectratio]{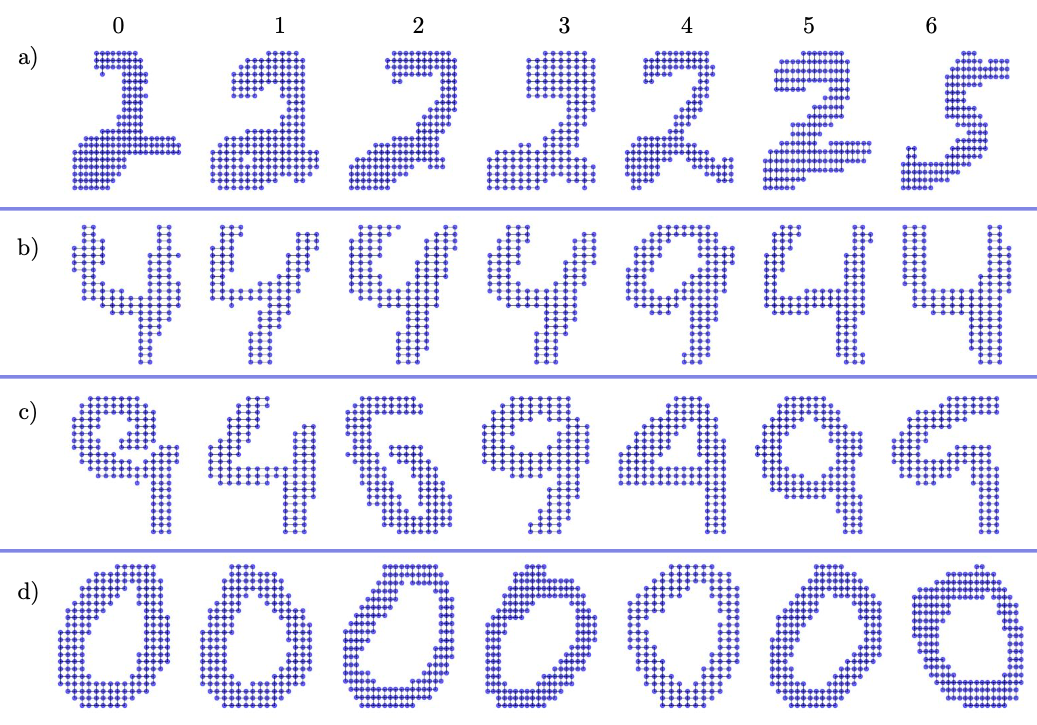}
    \caption{Six nearest neighbors of graphs in latent space from MNIST for four graphs. Column 0 correspond to the randomly chosen graph then the six nearest neighbors are draw in increasing distance order from left to right (from 1 to 6).}
    \label{fig:mnist_neighbors}
\end{figure}

\begin{figure}[h]
    \centering
    \includegraphics[width=9cm,height=29cm,keepaspectratio]{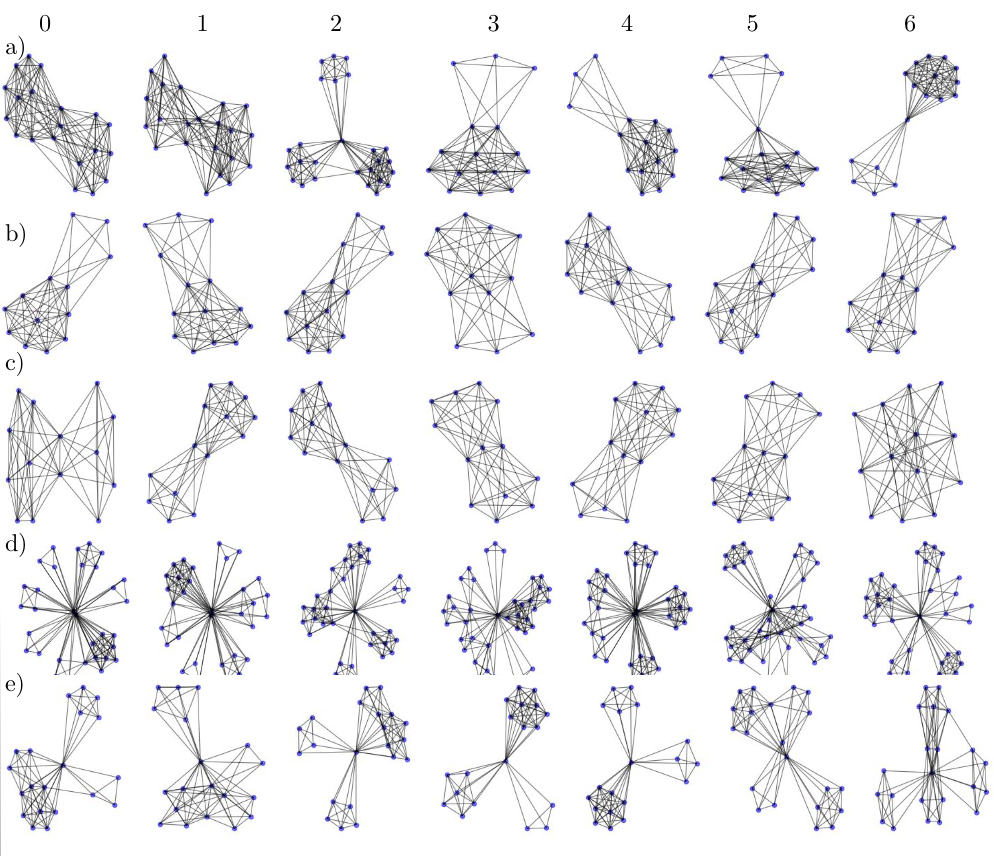}
    \caption{Six nearest neighbors of graphs in latent space from IMDB for ten graphs. Column 0 correspond to the randomly chosen graph then the six nearest neighbors are draw in increasing distance order from left to right (from 1 to 6).}
    \label{fig:imdb_neighbors}
\end{figure}

\begin{figure}[h]
    \centering
    \includegraphics[width=9cm,height=29cm,keepaspectratio]{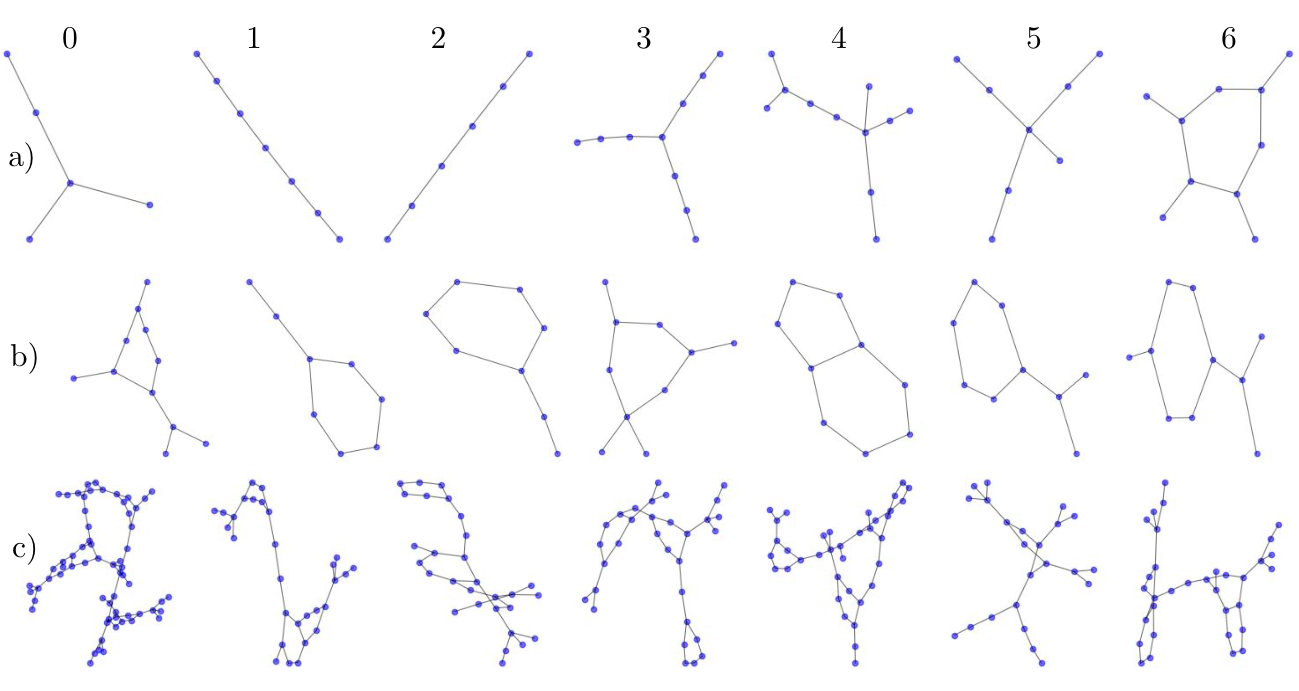}
    \caption{Six nearest neighbors of graphs in latent space from PTC for six graphs. Column 0 correspond to the randomly chosen graph then the six nearest neighbors are draw in increasing distance order from left to right (from 1 to 6).}
    \label{fig:ptc_neighbors}
\end{figure}

\section{Details about the datasets}
In this section we give additional details on some datasets we used.  
\subsection{DLA}
The DLA dataset has been artificially generated with Diffusion Limited Aggregation (\cite{witten1981diffusion}), a random process that creates cluster of particles following a Brownian motion. Particles are added one by one, and when two particles touch, they can aggregate with some probability of stickiness $p$. 
  
The resulting structure is a tree, each particle being a node and each link corresponding to a bond. 
The resulting graphs have scale free properties \cite{witten1983diffusion}. The degree distribution of the nodes and their position in space will depend on $p$. 
  
We generated a total of $1000$ graphs with $500$ nodes each. This dataset is splited into two classes, one with stickiness $p=1$ and the other with stickiness $p=0.05$. The attributes are the $x$ and $y$ coordinates of the particles following a 2D Brownian motion for simplicity. 
  
It has been observed that Graph2Vec is unable to reach a good accuracy by relying solely on node degree distribution, while Hierarchical Graph2Vec is able to use the features and reach near perfect accuracy. The code that generated this dataset can be found here: \url{https://github.com/Algue-Rythme/DiffusionLimitedAgregation}
\subsection{MNIST and USPS}
 We produce graphs from MNIST (resp. USPS) handwritten digits dataset . The graphs are created by removing all the pixel with luminosity equals to $0$ (resp. less than $0.3$), by mapping each remaining pixel to a node, and then by adding the $x$ and $y$ coordinate of each node to the vector of attributes. Due to the size of MNIST ($70,000$ images in total) we kept only the test split ($10,000$ images) to train the embeddings, and kept the whole dataset for USPS (9298 images). However these datasets and their graphs remain way larger than the other standard benchmarks.  

\end{document}